\documentclass[letterpaper]{article} 
\usepackage{aaai25}  
\usepackage{times}  
\usepackage{helvet}  
\usepackage{courier}  
\usepackage[hyphens]{url}  
\usepackage{graphicx} 
\usepackage{natbib}  
\usepackage{caption} 
\usepackage{algorithm}
\usepackage{algorithmic}

\usepackage{newfloat}
\usepackage{listings}
\DeclareCaptionStyle{ruled}{labelfont=normalfont,labelsep=colon,strut=off} 
\floatstyle{ruled}
\newfloat{listing}{tb}{lst}{}
\floatname{listing}{Listing}

\usepackage{algorithm}
\usepackage{algorithmic}
\usepackage{amsfonts}
\usepackage{amsthm}
\usepackage{amsmath}
\usepackage{amssymb}

\newtheorem{theorem}{Theorem}
\newtheorem{lemma}{Lemma}

\newtheorem{definition}{Definition}

\title{New Mechanisms in Flex Distribution for Bounded Suboptimal\\Multi-Agent Path Finding}
\author {
    Shao-Hung Chan\textsuperscript{\rm 1},
    Thomy Phan\textsuperscript{\rm 1},
    Jiaoyang Li\textsuperscript{\rm 2},
    Sven Koenig\textsuperscript{\rm 3}
}
\affiliations {
    \textsuperscript{\rm 1}University of Southern California\\
    \textsuperscript{\rm 2}Carnegie Mellon University\\
    \textsuperscript{\rm 3}University of California, Irvine\\
    shaohung@usc.edu, thomy.phan@usc.edu, jiaoyangli@cmu.edu, svenk@uci.edu 
}

\usepackage{bibentry}

\begin{document}

\maketitle

\begin{abstract}
Multi-Agent Path Finding (MAPF) is the problem of finding a set of collision-free paths, one for each agent in a shared environment. Its objective is to minimize the sum of path costs (SOC), where the path cost of each agent is defined as the travel time from its start location to its target location.
Explicit Estimation Conflict-Based Search (EECBS) is the leading algorithm for bounded-suboptimal MAPF, with the SOC of the solution being at most a user-specified factor $w$ away from optimal.
EECBS maintains sets of paths and a lower bound $LB$ on the optimal SOC. Then, it iteratively selects a set of paths whose SOC is at most $w \cdot LB$ and introduces constraints to resolve collisions.
For each path in a set, EECBS maintains a lower bound on its optimal path that satisfies constraints.
By finding an individually bounded-suboptimal path with cost at most a threshold of $w$ times its lower bound, EECBS guarantees to find a bounded-suboptimal solution.
To speed up EECBS, previous work uses \textit{flex distribution} to increase the threshold.
Though EECBS with flex distribution guarantees to find a bounded-suboptimal solution, increasing the thresholds may push the SOC beyond $w \cdot LB$, forcing EECBS to switch among different sets of paths (whose SOC are still at most $w \cdot LB$) instead of resolving collisions on a particular set of paths, and thus reducing efficiency.
To address this issue, we propose \textit{Conflict-Based Flex Distribution} that distributes flex in proportion to the number of collisions.
We also estimate the extra travel time (i.e., delays) needed to satisfy constraints and propose \textit{Delay-Based Flex Distribution}. 
On top of that, we propose \textit{Mixed-Strategy Flex Distribution}, combining both in a hierarchical framework.
We prove that EECBS with our new flex distribution mechanisms is complete and bounded-suboptimal.
Our experiments show that our approaches outperform the original (greedy) flex distribution.
Also, we redesign Focal-A* search from the previous work to improve $LB$ for a congested environment.
\end{abstract}

\begin{links}
    \link{Code}{https://github.com/shchan13/EECBS-MFD}
\end{links}

\section{Introduction}
Multi-Agent Path Finding (MAPF) is the problem of finding collision-free paths, one for each agent moving from its start location to its target location in a shared environment.
An optimal solution for a MAPF instance is a set of collision-free paths with a minimal \textit{sum of path costs} (SOC), where the path cost of an agent is its travel time for moving from its start location to its target location.
The MAPF applications include autonomous warehouses~\cite{WurmanAIM2008} and traffic management~\cite{LiAAAI23Intersection}.

Since finding optimal MAPF solutions is NP-hard~\cite{YuAAAI2013NPHard}, bounded-suboptimal MAPF algorithms have been used to speed up the search while still providing guarantees on the solution quality.
That is, the SOC is at most a user-specified suboptimality factor $w$ away from optimal.
One of the leading MAPF algorithms is Explicit Estimation Conflict-Based Search (EECBS)~\cite{LiAAAI2021EECBS}, which first finds a path for each agent individually and then resolves the collisions.
EECBS maintains several sets of paths, where each set of paths is associated with a set of constraints that have been used to resolve collisions.
For each path in a set of paths, EECBS maintains a lower bound on its optimal path that satisfies constraints.
Thus, given a set of paths, the sum of each path's lower bound (SOLB) is the lower bound on the SOC of the optimal solution that satisfies the constraints, and the minimum SOLB over all sets of paths is a lower bound $LB$ on the SOC of the optimal solution.
EECBS iteratively selects a set of paths whose SOC is at most $w \cdot LB$ and introduces new constraints to resolve collisions.
By finding an individually bounded-suboptimal path whose cost is at most a \textit{threshold} equal to $w$ times its lower bound, EECBS guarantees that each set of paths is bounded-suboptimal with respect to their constraints, i.e., the SOC is at most $w$ away from the optimal solution that satisfies the constraints.
However, the requirement that each path needs to be individually bounded-suboptimal prohibits EECBS from finding solutions where the paths are not all individually bounded-suboptimal, but their SOC is still bounded-suboptimal with respect to the constraints.
Thus, previous work has proposed using \textit{flex distribution} to relax this requirement while still guaranteeing to find a bounded-suboptimal solution~\cite{ChanAAAI2022}.
Before finding a path for an agent, the approach sums the differences between $w$ times the lower bounds and the path costs over all other agents, defined as the \textit{flex}, and then increases the threshold of the agent by this flex.
The increased threshold allows the agent to take a longer path to avoid collisions with other agents while still guaranteeing that the solution is bounded-suboptimal.

EECBS with flex distribution is able to explore part of the solution space that is not reachable when each path has to be individually bounded-suboptimal.
However, using all the flex to increase the thresholds may push the SOC beyond $w \cdot LB$, which forces EECBS to switch among different sets of paths whose SOCs are still at most $w \cdot LB$ instead of focusing on resolving collisions in a particular set of paths.
On the other hand, if the path satisfying constraints for an agent has a cost beyond $w$ times its lower bound, then by using flex, EECBS may no longer have to increase its lower bound to find the path as the flex increases its threshold.
In this case, with flex distribution, the SOLB of a set of paths may be lower than that without, which may result in an under-estimated $LB$.
That is, EECBS with flex distribution may overlook solutions that are bounded-suboptimal but whose SOCs are larger than a factor $w\cdot LB$.
Thus, in this paper, we aim to address the issue of flex distribution that the sets of paths may have SOCs larger than $w \cdot LB$.
Our contributions are as follows:
\begin{itemize}
    \item To eliminate the increment of SOC in a set of paths, we propose new mechanisms for flex distribution to avoid using all the flex.
    \item For a congested environment, we redesign \textit{Focal-A*} search from~\citet{ChanAAAI2022} for EECBS to find a path for an agent while trying to increase its lower bound in congested environments.
\end{itemize}%
Our empirical evaluation shows that EECBS with our approaches improves the success rate within a runtime limit of 120 seconds in comparison to the state-of-the-art EECBS.
We also provide a case study to demonstrate our approach in comparison to the state-of-the-art EECBS.

\section{Preliminary}

\subsection{Multi-Agent Path Finding}

We use the MAPF definition by~\citet{SternSoCS19MAPFDef}. A MAPF instance consists of an undirected graph $G=(V, E)$ and a set of $k$ agents $\bigcup_{i=1}^{k}\{a_i\}$. Each agent $a_i$ has a unique start vertex $s_i$ and a unique target vertex $l_i$.
In this paper, we use four-connected undirected grid graphs, and time is discretized into timesteps. At each timestep, an agent can either move to an adjacent vertex or wait at its current vertex.
A \textit{path} of an agent, starting at its start vertex and ending at its target vertex, is a sequence of vertices indicating where the agent is at each timestep.
Each agent permanently waits at its target vertex after it completes its path. The \textit{cost} of a path is the number of timesteps needed for the agent to move from its start to target vertices, ignoring those when the agent permanently waits at its target vertex.
When a pair of agents $a_i$ and $a_j$ respectively follow their paths $p_i$ and $p_j$, a \textit{vertex conflict}, denoted as $\langle a_i, a_j, v, t \rangle$, occurs iff both agents reach the same vertex $v$ at the same timestep $t$. On the other hand, an \textit{edge conflict}, denoted as $\langle a_i, a_j, u, v, t \rangle$, occurs iff these two agents traverse the same edge $(u,v) \in E$ in opposite directions at the same timestep $t$.
A \textit{solution} is a set of conflict-free paths, one for each agent. A solution is \textit{optimal} iff its \textit{sum of path costs} (SOC) is minimum, denoted as $C^{*}$, and \textit{bounded-suboptimal} iff its SOC is at most $w \cdot C^{*}$, where $w \geq 1$ is a user-specified\textit{ suboptimality factor}.


\subsection{Explicit Estimation Conflict-Based Search}

Explicit Estimation Conflict-Based Search (EECBS)~\citep{LiAAAI2021EECBS} is a two-level bounded-suboptimal algorithm for MAPF. Its strategy is to iteratively resolve a conflict by introducing constraints on the high level and then finding the paths on the low level to satisfy the constraints.
On the high level, EECBS constructs a \textit{Constraint Tree} (CT). A \textit{vertex constraint} indicates that an agent is not allowed to reach a vertex at a particular timestep, and an \textit{edge constraint} indicates an agent is not allowed to traverse an edge at a particular timestep.
A \textit{CT node} $N$ contains a set of constraints $\bigcup_{i=1}^{k} \Psi_i(N)$, where $\Psi_i(N)$ is the set of (vertex and edge) constraints corresponding to agent $a_i$.
A CT node also contains a set of paths, one for each agent, that satisfy its sets of constraints.
EECBS then runs a bounded-suboptimal heuristics search algorithm on the high level called Explicit Estimation Search (EES)~\citep{ThayerIJCAI2011EES}. It maintains a set of lists that contain all the generated but not yet expanded CT nodes, denoted as LISTs.

At each iteration, EECBS selects a CT node $\hat{N}$ and selects a conflict, if it exists, between a pair of paths in $\hat{N}$.
EECBS then expands CT node $\hat{N}$ and generates two child CT nodes $N$ and $N'$.
Each child CT node respectively contains not only the set of constraints from its parent CT node $\hat{N}$ but also an additional constraint $\psi_i$ (with respect to $\psi_j$) that prevents agent $a_i$ (with respect to $a_j$) from occupying vertex/edge at the conflicting timestep, i.e.,
\begin{equation}
    \Psi_i(N) = \Psi_i(\hat{N}) \cup \{\psi_i\};\:
    \Psi_j(N') = \Psi_j(\hat{N}) \cup \{\psi_j\}.
\end{equation}%
With the help of the low-level search, EECBS finds a path for agent $a_i$ (with respect to $a_j$) that satisfies constraints $\Psi_i(N)$ (with respect to $\Psi_j(N')$). EECBS pushes the child CT node to the LISTs if the path is found.

On the low level, to find a path for an agent $a_i$ in a CT node $N$, EECBS constructs a search tree with each vertex-timestep (v-t) node $n$ containing a tuple $(v,t)$ that indicates an agent staying at vertex $v$ at timestep $t$.
For a v-t node $n=(v,t)$, we define a priority function $f_i(n) = g_i(n) + h_i(v)$, where $g_i(n) = t$ is the number of timesteps for agent $a_i$ to move from its start vertex $s_i$ to vertex $v$ and $h_i(v)$ is an admissible heuristic that underestimates the number of timesteps needed to move from vertex $v$ to its target vertex $l_i$ while satisfying constraints $\Psi_i(N)$.
The number of conflicts $x_i(n)$ is computed with the paths of the other agents when agent $a_i$ moves from v-t nodes $(s_i, 0)$ to $(v,t)$.

EECBS runs \textit{focal search}~\citep{PearTPAMIl1982FocalSearch} to find a bounded-suboptimal path with its cost $c_i(N)$ and lower bound $lb_i(N)$, which is the lower bound on the minimum cost of the path satisfying constraints $\Psi_i(N)$.
Focal search maintains two lists: OPEN$_L$ and FOCAL$_L$.
OPEN$_L$ sorts all the generated but not yet expanded v-t nodes $n$ in increasing order of priority function $f_i(n)$.
FOCAL$_L$ contains those v-t nodes in OPEN$_L$ whose $f_i(n)$ are less than or equal to a \textit{threshold}, defined as
\begin{equation}
    \tau_i = w \cdot f_{\min,i}(N),
    \label{eq:ori_thre}
\end{equation}%
where $f_{\min,i}(N)$ is the minimum $f$-value among all v-t nodes in OPEN$_L$ in CT node $N$.
Focal search sorts these v-t nodes $n \in \text{FOCAL}_L$ in increasing order of their numbers of conflicts $x_i(n)$.

At each iteration, focal search includes more v-t nodes into FOCAL$_L$ if $f_{\min,i}(N)$ increases, and then expands the top v-t node in FOCAL$_L$ that has the minimum $x_i$ value.
When focal search terminates with the vertex of the top v-t node being the target vertex $l_i$, we set $lb_i(N)$ to $f_{\min,i}(N)$, which is defined as the (best-known) lower bound on the minimum cost of the path that satisfies constraints $\Psi_i(N)$.
Since $f_i(n)$ of any v-t node $n$ in FOCAL$_L$ is at most threshold $\tau_i$, focal search always finds a bounded-suboptimal path with cost $c_i(N)$ satisfying $lb_i(N) \leq c_i(N) \leq \tau_i$.
The SOC of CT node $N$ thus satisfies
\begin{equation}
    C(N) = \sum_{j \in [k]} c_i(N) \leq w \cdot \sum_{j \in [k]} lb_i(N) = w \cdot LB(N).
    \label{eq:bounded_CT}
\end{equation}%
EES keeps track of the minimum SOLB $LB(N)$ among all the CT nodes $N$ in LISTs, which indicates the lower bound $LB$ on the SOC of the optimal solution.
Accordingly, we provide the following definitions and a lemma:
\begin{definition}
    The \textbf{local suboptimality} of a CT node $N$ is defined as $C(N) / LB(N)$.
\end{definition}%
\begin{definition}
    A CT node $N$ is \textbf{locally bounded-suboptimal} iff its local suboptimality is bounded by the suboptimality factor $w$, i.e., Equation (\ref{eq:bounded_CT}) holds.
\end{definition}%
\begin{lemma}[\citet{LiAAAI2021EECBS}]
EECBS is complete and bounded-suboptimal if each generated CT node is locally bounded-suboptimal.
\label{lem:complete}
\end{lemma}%
\begin{definition}
    The \textbf{global suboptimality} of a CT node $N$ is defined as $C(N) / LB$, where $LB := \min_{\tilde{N} \in \text{LISTs}} LB(\tilde{N})$.
\end{definition}%
\begin{definition}
    A CT node $N$ is \textbf{globally bounded-suboptimal} iff its global suboptimality is bounded by the suboptimality factor $w$, i.e., $C(N) / LB \leq w$ holds.
\end{definition}%

At each iteration, EES only expands a CT node $N$ from LISTs that are globally bounded-suboptimal in order to guarantee finding a bounded-suboptimal solution.

\subsection{Greedy-Based Flex Distribution (GFD)}
EECBS finds a bounded-suboptimal solution as long as each generated CT node $N$ is locally bounded-suboptimal.
Such a guarantee is maintained by running the focal search that finds an individually bounded-suboptimal path for each agent, i.e., $c_i(N) \leq w \cdot lb_i(N), \forall i \in [k]$.
However, one can observe that it is unnecessary to insist that each path be individually bounded-suboptimal in order to make a CT node locally bounded-suboptimal.
Thus, previous work had proposed \textit{flex distribution} that increases thresholds when finding paths while still guaranteeing finding a bounded-suboptimal solution~\cite{ChanAAAI2022}.
At each iteration, suppose EECBS expands CT node $\hat{N}$ and generates one of its child CT nodes $N$ with an additional constraint on agent $a_i$.
While the paths of other agents remain fixed (i.e., $lb_j(\hat{N}) = lb_j(N)$ and $c_j(\hat{N}) = c_j(N)$ hold for all $j \in [k] \setminus \{i\}$), to find the path for agent $a_i$ in a CT node $N$ via focal search, \citet{ChanAAAI2022} increases its threshold by adding the \textit{flex} over the other $k-1$ agents, denoted as $\Delta_i$.
Modified from Equation (\ref{eq:ori_thre}), the threshold thus becomes
\begin{equation}
    \tau_i = w \cdot \max\{f_{\min,i}(N), lb_i(\hat{N})\} + \Delta_i.
    \label{eq:flex_threshold}
\end{equation}%
\citet{ChanAAAI2022} sets $\Delta_i$ to the \textit{maximum allowed flex} $\Delta_{\max,i}$, defined as
\begin{equation}
    \Delta_i = \Delta_{\max,i} = \sum_{j \in [k] \setminus \{i\}} (w \cdot lb_j(\hat{N}) -  c_j(\hat{N})).
    \label{eq:flex_ub}
\end{equation}%
That is, \citet{ChanAAAI2022} distributes the maximum allowed flex, which increases the threshold if $\Delta_{\max,i} \geq 0$ holds. The work also proves that the threshold $\tau_i \geq lb_i(N)$ always holds even when the maximum allowed flex is negative. In this paper, we term such an approach \textit{Greedy-Based Flex Distribution} (GFD).
If $\Delta_{\max,i} < 0$ occurs, then at least one of the other agents $a_{j \in [k] \setminus \{i\}}$ has the path with cost $c_j > w \cdot lb_j(\hat{N})$, indicating that the flex from agent $a_i$ was used along the branch of CT node $N$.
Also, when focal search terminates, lower bound $lb_i(N)$ in CT node $N$ is set to
\begin{equation}
    lb_i(N) = \max\{f_{\min,i}(N), lb_i(\hat{N})\},
    \label{eq:lb}
\end{equation}%
which is the \textit{best-known lower bound} on the cost of the optimal path for agent $a_i$ that satisfies the set of constraints $\Psi_i(N)$ in CT node $N$. This is to ensure the flex provided by the agent does not decrease.
%
%
\begin{theorem}
    EECBS with GFD is complete and bounded-suboptimal.
    \label{thm:GFD}
\end{theorem}%

\begin{proof}
After focal search finds a path for agent $a_i$ via GFD, the SOC of the generated CT node $N$ satisfies
\begin{equation}
\begin{split}
    &C(N) = c_i(N) + \sum_{j \in [k] \setminus \{i\}} c_j(\hat{N})\leq \tau_i + \sum_{j \in [k] \setminus \{i\}} c_j(\hat{N})\\
    &\overset{(\ref{eq:flex_threshold})}= w \cdot \max\{f_{\min,i}(N),lb_i(\hat{N})\} + \Delta_i + \sum_{j \in [k] \setminus \{i\}} c_j(\hat{N})\\
    &\overset{(\ref{eq:flex_ub}),(\ref{eq:lb})}= w \cdot lb_i(N) + \Delta_{\max,i} + \sum_{j \in [k] \setminus \{i\}} c_j(\hat{N})\\
    &= w \cdot \sum_{j \in [k]} lb_j(N) = w \cdot LB(N).
    \label{eq:proof_CT}
\end{split}
\end{equation}%
Thus, each CT node is locally bounded-suboptimal. According to Lemma \ref{lem:complete}, EECBS with GFD is complete and bounded-suboptimal.
\end{proof}%


\section{Conflict-Based Flex Distribution (CFD)}

\subsubsection{GFD Limitations}

\begin{figure}[t]
    \centering
    \includegraphics[width=0.81\linewidth]{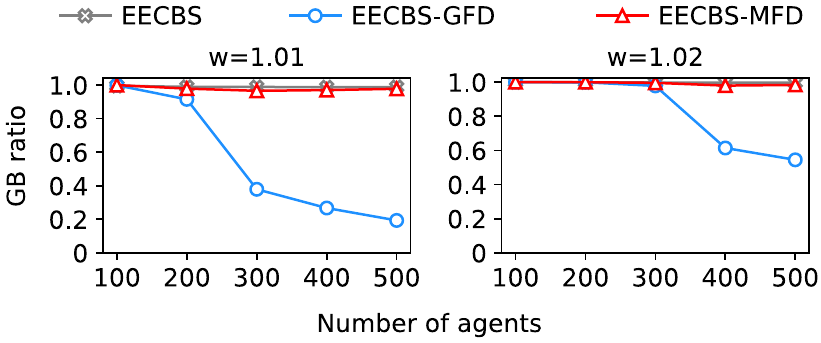}
    \caption{The GB ratio of EECBS, EECBS-GFD, and EECBS-MFD when solving MAPF instances on \texttt{ost003d} graph with $w=\{1.01, 1.02\}$.}
    \label{fig:GFD_limitaion_globRatio}
\end{figure}%

With the relaxation of the threshold on the path cost provided by GFD, EECBS can further reduce the number of conflicts when finding a path for an agent.
However, one limitation of GFD is that agents may use all the flex when finding their paths in a CT node $N$, resulting in an increment in its SOC $C(N)$ close to $w \cdot LB(N)$.
In this case, though each generated CT node is still locally bounded-suboptimal, it may not be globally suboptimal and thus may not be expanded.
Thus, EECBS may switch among CT nodes on different branches whose SOCs are at most $w \cdot LB$ instead of resolving collisions along the branch of CT node $N$.
Figure~\ref{fig:GFD_limitaion_globRatio} shows the \textit{GB ratio}, defined as the ratio of CT nodes that are globally bounded-suboptimal with respect to the lower bound $LB$ when they are generated, indicating EECBS with GFD tends to generate CT nodes that are not globally bounded-suboptimal and thus lowers the efficiency.
On the other hand, if the path satisfying constraints $\Psi_i(N)$ for agent $a_i$ has a cost larger than $w \cdot lb_i(N)$, then with the increasing threshold, focal search may no longer have to expand v-t nodes whose $f$-values are close to optimal to increase its lower bound in finding such a path.
Thus, the lower bound $LB(N)$ of CT node $N$ with flex distribution may be lower than that without, resulting in an under-estimated $LB$.
That is, EECBS with flex distribution may overlook bounded-suboptimal solutions whose SOCs are larger than $ w\cdot LB$, and thus affect the performance in congested graphs~\cite{ChanAAAI2022}.

\subsubsection{Key Observation} When finding a path for an agent, the objective (or the ``need'') for the focal search is to quickly minimize the number of conflicts while satisfying constraints. This observation provides a clue for us in designing new mechanisms that target the needs of the focal search.

To avoid using all the flex, we propose \textit{Conflict-Based Flex Distribution} (CFD) that distributes flex fractionally according to the number of conflicts involved with the agent that needs to find its path.
Suppose EECBS expands CT node $\hat{N}$ and generates one of its child CT nodes $N$ that has to find a path for agent $a_i$ satisfying constraints $\Psi_i(N)$.
If the maximum allowed flex is non-negative, then we distribute flex in proportion to the number of conflicts that it encounters with other agents. Otherwise, we switch to GFD. Thus, the distributed flex via CFD becomes
\begin{equation}
    \Delta_i = 
    \begin{cases}
    \rho_i \cdot \Delta_{\max,i}, \text{where } \rho_i = \frac{\mathcal{X}_i(\hat{N})}{\mathcal{X}(\hat{N})},& \text{if } \Delta_{\max,i} \geq 0\\
    \Delta_{\max.i},              & \text{otherwise},
    \end{cases}
    \label{eq:CFD}
\end{equation}%
where $\mathcal{X}_i(\hat{N})$ is the number of conflicts that the path of agent $a_i$ encounters with the paths of other agents $a_{j \in [k] \setminus \{i\}}$ in CT node $\hat{N}$, and $\mathcal{X}(\hat{N})$ is the number of conflicts among all pairs of paths in CT node $\hat{N}$.
%
%

The rationale of CFD is to (implicitly) resolve the conflicts during focal search by using flex instead of (explicitly) introducing constraints when finding a path for agent $a_i$ in CT node $N$. Also, compared to GFD, which uses as much flex as possible, CFD saves some flex in the future for other paths to resolve conflicts.
On the other hand, if the maximum flex is negative, then we return to GFD since this is the case where, along the branch of CT node $N$, some other agents $a_{j \in [k] \setminus \{i\}}$ have already used flex from agent $a_i$. Thus, we should reduce the threshold $\tau_i$ to ensure CT node $N$ is still bounded-suboptimal (i.e., Equation (\ref{eq:bounded_CT}) holds) after finding a path $p_i(N)$ for agent $a_i$.
Although this distribution method is trivial, we found out that it efficiently improves the efficiency of EECBS, as shown in Figure~\ref{fig:succ_all}.


\section{Delay-Based Flex Distribution (DFD)}

Although the original motivation of flex distribution is to reduce the number of conflicts, we can further utilize flex to speed up the search while satisfying constraints.
Suppose EECBS expands CT node $\hat{N}$ and generates one of its child CT nodes $N$ that has to find a path for agent $a_i$ to satisfy constraints $\Psi_i(N)$.
In this case, constraints typically cause delay as focal search may have to find a longer path.
We define the \textit{delay} as the difference in the path costs for agent $a_i$ between CT nodes $N$ and $\hat{N}$.
Thus, on top of CFD, we propose \textit{Delay-Based Flex Distribution} (DFD) that arranges some flex $\Delta_{d,i}$ by estimating the delay in finding a path for agent $a_i$ that satisfies constraints $\Psi_i(N)$. That is,
\begin{equation}
    \Delta_{d,i} = \min\{\Delta_{\max,i} \:, \sum_{\psi \in \Psi_i(N)} d_{\psi}\},
\end{equation}%
where $d_{\psi}$ is the estimated delay needed for a path to satisfy a constraint $\psi \in \Psi_i(N)$, and the value of $\Delta_{d,i}$ is upper-bounded by the maximum allowed flex in CT node $\hat{N}$.
Also, we require the sum of delays to be non-negative, i.e., $\sum_{\psi \in \Psi_i(N)} d_{\psi} \geq 0$ holds, resulting in $\Delta_{d,i} \geq 0$ if $\Delta_{\max,i} \geq 0$ holds.
Combined with CFD, the distributed flex is thus
\begin{equation}
    \Delta_i = 
    \begin{cases}
    \Delta_{d,i} + \rho_i \cdot (\Delta_{\max,i} - \Delta_{d,i}),& \text{if } \Delta_{\max,i} \geq 0\\
    \Delta_{\max.i},              & \text{otherwise.}
    \end{cases}
    \label{eq:DFD}
\end{equation}%
The rationale is to distribute some flex $\Delta_{d,i}$ for the delays from the constraints and then distribute the remaining flex $\Delta_{\max,i} - \Delta_{d,i}$ with a ratio of $\rho_i$ from CFD.

However, how to estimate the delay precisely without introducing computational overhead remains an open question, and in this paper, we propose a rule-based strategy.
For each vertex or edge constraint, we assume the replanned path can satisfy such a constraint by introducing a wait action, resulting in a delay of 1.

\subsubsection{Delays from Symmetric Reasoning} Besides the vertex and edge constraints, previous works have proposed reasoning for conflicts based on the graph's geometry, known as symmetric breaking for EECBS~\cite{LiICAPS2020TargetCorridorReasoning, LiAAAI2021EECBS}.
In particular, the works identified the \textit{corridor conflict} and the \textit{target conflict}.
Given a graph $G=(V,E)$, a corridor $Cr = Cr_0 \cup \{v_b, v_e\}$ is defined as a chain of connected vertices $Cr_0 \subseteq V$, each with a degree of 2, together with two endpoints $\{v_b, v_e\} \subseteq V$ connected to $Cr_0$. A corridor conflict occurs if two agents traverse the same corridor in opposite directions and reach the same vertex/edge at the same timestep.
Suppose agent $a_i$ traverses from vertices $v_b$ to $v_e$ and agent $a_j$ in the opposite direction. To resolve the corridor conflict from a CT node $\hat{N}$, when generating one of its child CT nodes $N$, \citet{LiICAPS2020TargetCorridorReasoning} introduced a \textit{range constraint} that forbids agent $a_i$ from reaching its exit vertex $v_e$ in the range of timestep $[0, t_{\min,j}]$, where $t_{\min,j}$ is the minimum timestep needed for agent $a_j$ to traverse to its exit vertex $v_b$. Similarly, a range constraint is introduced to agent $a_j$ when generating the other child CT node.
Thus, given a range constraint $\psi$, we intuitively estimate its delays as $d_{\psi} = t_{\min,j} + 1 - t_{e,i}$, where $t_{e,i}$ is the timestep agent $a_i$ reaches its exit vertex $v_e$ when following its conflicting path in CT node $\hat{N}$.

On the other hand, a target conflict occurs if one agent $a_i$ traverses the target vertex $l_j$ of another agent $a_j$ at the timestep $t$ when agent $a_j$ has already arrived at its target vertex $l_j$ at timestep $t_j$ and stayed there permanently, i.e., $t_j = c_j(\hat{N}) \leq t$. 
To resolve the target conflict from a CT node $\hat{N}$, \citet{LiICAPS2020TargetCorridorReasoning} introduced a \textit{length constraint}. That is, one of its child CT nodes $N$ has the constraint where agent $a_j$ should reach its target vertex with the path cost $c_j(N) \leq t$, indicating that other agents $a_{m \in [k] \setminus \{j\}}$ should not traverse vertex $l_j$ at or after timestep $t$.
When finding the path for any agent violating the constraint, since introducing delay may cause it to reach vertex $l_j$ even later, we assume that 0 delay is needed.
On the other hand, the other child CT node $N'$ has the constraint where the agent $a_j$ should find a path with cost $c_j(N') > t$ (i.e., reaching its target vertex later than $t$).
When finding the path for agent $a_j$, we set the delay to $d_{\psi} = t - c_j(\hat{N})$.
Note that since EECBS utilizes symmetric reasoning during the search for better performance, we can obtain the estimated delay without introducing a huge computational overhead.
Figure~\ref{fig:DFD_example} shows an illustrative example of the variable relations when using DFD with the positive maximum allowed flex.

\begin{theorem}
    EECBS with CFD and/or DFD is complete and bounded-suboptimal.
\end{theorem}%
\begin{proof}
If $\Delta_{\max,i} < 0$ holds, we follow the proof in Theorem \ref{thm:GFD} since we switch to GFD for both CFD and DFD. So, here we only focus on the case where $\Delta_{\max,i} \geq 0$ holds. For CFD, it is trivial since the set of conflicts that agent $a_i$ encounters with other agents is a subset of all the conflicts among the paths in the CT node.
Thus, $\rho_i \in [0,1]$ holds, indicating that $\Delta_i \leq \Delta_{\max,i}$.
For DFD, since the value of $\Delta_{d,i}$ is upper-bounded by the maximum allowed flex, the first row of Equation (\ref{eq:DFD}) becomes
\begin{equation}
    \Delta_i = (1 - \rho_i) \cdot \Delta_{d,i} + \rho_i \cdot \Delta_{\max,i} \leq \Delta_{\max,i},
    \label{eq:another_DFD}
\end{equation}%
indicating that the maximum flex that DFD can distribute is $\Delta_{\max,i}$.
Equation (\ref{eq:another_DFD}) can also be interpreted as the sum of the flex $\rho_i \cdot \Delta_{\max,i}$ from CFD and $(1 - \rho_i) \cdot \Delta_{d,i}$ from the delay estimation.
In this case, CFD can be viewed as a special case of DFD where $\Delta_{d,i} = 0$.
Thus, for CFD and DFD, the generated CT node $N$ is still bounded-suboptimal according to Equations (\ref{eq:flex_ub}) and (\ref{eq:proof_CT}), resulting in EECBS being complete and bounded-suboptimal according to Lemma \ref{lem:complete}.
\end{proof}%

\begin{figure}
    \centering
    \includegraphics[width=0.75\linewidth]{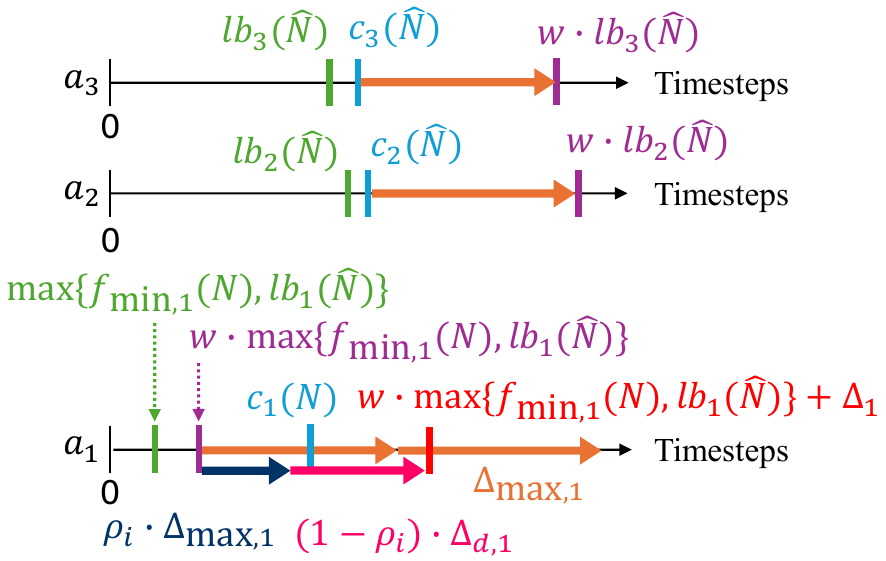}
    \caption{Suppose EECBS-DFD expands CT node $\hat{N}$ and generates one of its child CT nodes $N$ that replans the path of agent $a_1$ with positive maximum allowed flex $\Delta_{\max,1}$ over agents $a_2$ and $a_3$.
    The orange solid arrows are the maximum allowed flex, and the blue and pink ones are the distributed flex from DFD.}
    \label{fig:DFD_example}
\end{figure}%


\section{Mixed-Strategy Flex Distribution (MFD)}

Although CFD and DFD eliminate the amount of distributed flex, there may still exist a situation where some agents utilize too much flex and thus increase the SOC.
Thus, based on CFD and DFD, we also propose a \textit{Mixed-Strategy Flex Distribution} (MFD).
Suppose that EES expands CT node $\hat{N}$ and generates one of its child CT nodes $N$ with constraints $\Psi_i(N)$ on agent $a_i$.
If the maximum allowed flex is negative, then MFD switches to GFD.
Otherwise, to find a path for agent $a_i$ in CT node $N$, MFD first uses DFD to calculate the flex $\Delta_i$ (i.e., Equation (\ref{eq:DFD})).
This provides an under-estimation of its threshold, i.e., $w \cdot lb_i(\hat{N}) + \Delta_i$.
Then, MFD checks whether the sum of the threshold and the SOC from other agents is still globally bounded-suboptimal.
That is, if
\begin{equation}
    w \cdot lb_i(\hat{N}) + \Delta_i + \sum_{j \in [k] \setminus \{i\}} c_j(N) \leq w \cdot LB
    \label{eq:MFD}
\end{equation}%
holds, then the flex $\Delta_i$ is used for finding the path for agent $a_i$.
Otherwise, MFD switches to CFD to calculate the flex via Equation (\ref{eq:CFD}), i.e., set $\Delta_{d,i}$ to zero.
Again, MFD then checks whether Equation (\ref{eq:MFD}) holds. If so, then the flex is used.
Otherwise, MFD tries to calculate the flex via the CT node $N_F$ with the minimum SOLB in LISTs, i.e, the CT node $N_F$ with $LB(N_F) = LB$.
If both equations,
\begin{equation}
    \sum_{j \in [k] \setminus \{i\}} lb_j(N_F) < \sum_{j \in [k] \setminus \{i\}} lb_j(N)
    \label{eq:MFD_cond1}
\end{equation}%
and
\begin{equation}
    \sum_{j \in [k] \setminus \{i\}} c_j(N) < w \cdot \sum_{j \in [k] \setminus \{i\}} lb_j(N_F),
    \label{eq:MFD_cond2}
\end{equation}%
hold, then MFD modifies the maximum allowed flex as
\begin{equation}
    \Delta_{\max,i} = w \cdot \sum_{j \in [k] \setminus \{i\}} lb_j(N_F) - \sum_{j \in [k] \setminus \{i\}} c_j(N).
\end{equation}%
In this case, the condition from Equation (\ref{eq:MFD_cond1}) ensures that the maximum allowed flex computed from CT node $N_F$ is lower than that computed from CT node $N$, and the condition from Equation (\ref{eq:MFD_cond2}) ensures that the maximum allowed flex remains positive.
Also, if $lb_i(N) = lb_i(\hat{N}) = lb_i(N_F)$ holds, then the SOC of CT node $N$ satisfies
\begin{equation}
\begin{split}
    C(N) &\leq \sum_{j \in [k] \setminus \{i\}} c_j(N) + w \cdot lb_i(N) + \Delta_i\\
         &\leq \sum_{j \in [k] \setminus \{i\}} c_j(N) + w \cdot lb_i(N_F) + \Delta_{\max,i}\\
         &\leq \sum_{j \in [k] \setminus \{i\}} c_j(N) + w \cdot lb_i(N_F) +\\
         &\qquad\sum_{j \in [k] \setminus \{i\}} (w \cdot lb_j(N_F) - c_j(N)) = w \cdot LB,
\end{split}
\end{equation}%
indicating the CT node $N$ is globally bounded-suboptimal and thus has a chance to be expanded in future iterations.
Otherwise, as long as one of the equations does not hold, then MFD distributes zero flex.
Since MFD further eliminates the amount of distributed flex if the maximum allowed flex is positive and switches to GFD otherwise, it is still guaranteed that EECBS with MFD is complete and bounded-suboptimal.
Figure~\ref{fig:MFD_flowchart} shows how MFD calculate flex $\Delta_i$ when finding a path for agent $a_i$ satisfying constraints $\Psi_i(N)$ in CT node $N$.

\begin{figure}
    \centering
    \includegraphics[width=0.6\linewidth]{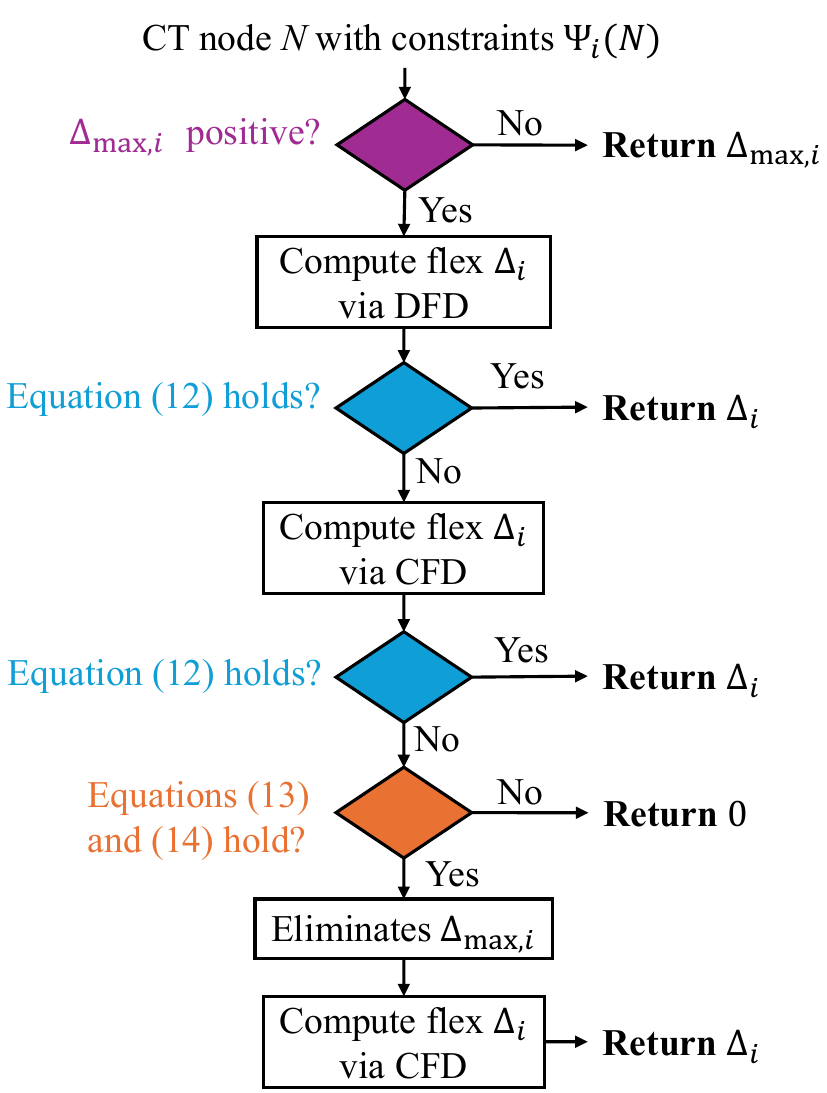}
    \caption{The flow chart of calculating flex via MFD.}
    \label{fig:MFD_flowchart}
\end{figure}%


\section{Re-design Low-level Focal-A* Search (FA*)}
Since GFD suffers from poor lower bound improvements when solving MAPF instances in congested graphs (e.g., tens of agents in a graph of size $32 \times 32$), we adapt \textit{Focal-A*} (FA*) search from~\citet{ChanAAAI2022}, which switches from focal search to A* search on the low level if the number of expansions exceeds a hyperparameter multiplying the number of expansions last time the previous search used to find the path of the same agent along the CT branch.
However, such an approach complicates the search and requires fine-tuning of the hyperparameter.
Thus, we simplify the previous FA* search such that it (\textbf{i}) expands v-t nodes from FOCAL$_L$ until a path is found, and then (\textbf{ii}) expands v-t nodes from OPEN$_L$ until either an optimal path is found or the current v-t node has its $f$-value larger than the cost of the path from step (\textbf{i}). After that, our FA* search returns the path from step (\textbf{i}) and the lower bound from step (\textbf{ii}). In this case, FA* search runs its search on the same search tree but expands v-t nodes with different priorities (between FOCAL$_L$ and OPEN$_L$) and spends more effort improving its lower bound via expanding v-t nodes from OPEN$_L$.
The difference between~\citet{ChanAAAI2022} and ours is that the former returns an optimal path if the number of v-t node expansions exceeds a threshold, while the latter returns the path same as focal search, but with a lower bound being the optimal path cost that satisfies the constraints, which is at least the lower bound from the focal search.


\section{Empirical Evaluation}
From the MAPF benchmark suite~\cite{SternSoCS19MAPFDef}, we evaluate our approach on six large four-connected grid graphs: a \texttt{city} graph (\textit{Boston\_0\_256}) of size $256 \times 256$, two game graphs, which are \texttt{den520d} of size $256 \times 257$ and \texttt{ost003d} of size $194 \times 194$, and a \texttt{warehouse} graph (\textit{warehouse-10-20-10-2-1}) of size $63 \times 161$ with corridor widths of one.
We conduct experiments on the available 25 random scenarios under each number of agents, where the numbers of agents tested are listed on the $x$-axis of Figure~\ref{fig:succ_all}.
For MAPF instances with more than 1000 agents, we generate agents with random start and target vertices in addition to the existing 1000 agents from each random scenario.
We set the runtime limit of 120 seconds unless mentioned otherwise.
We also implement enhancements including Bypassing Conflict~\cite{BoyarskiIJCAI15ICBS}, Prioritize Conflict~\cite{BoyarskiIJCAI15ICBS}, and Symmetric Reasoning~\cite{LiICAPS2020TargetCorridorReasoning}.
For bypassing conflicts with flex distribution, we adopt the mechanism from~\citet{ChanAAAI2022}.
All experiments are run on CentOS Linux, Intel Xeon 2640v4 CPUs, and 64 GB RAM.

\subsection{Performance Comparison}\label{sec:performance}
\begin{figure*}[t]
    \centering
    \includegraphics[width=0.85\linewidth]{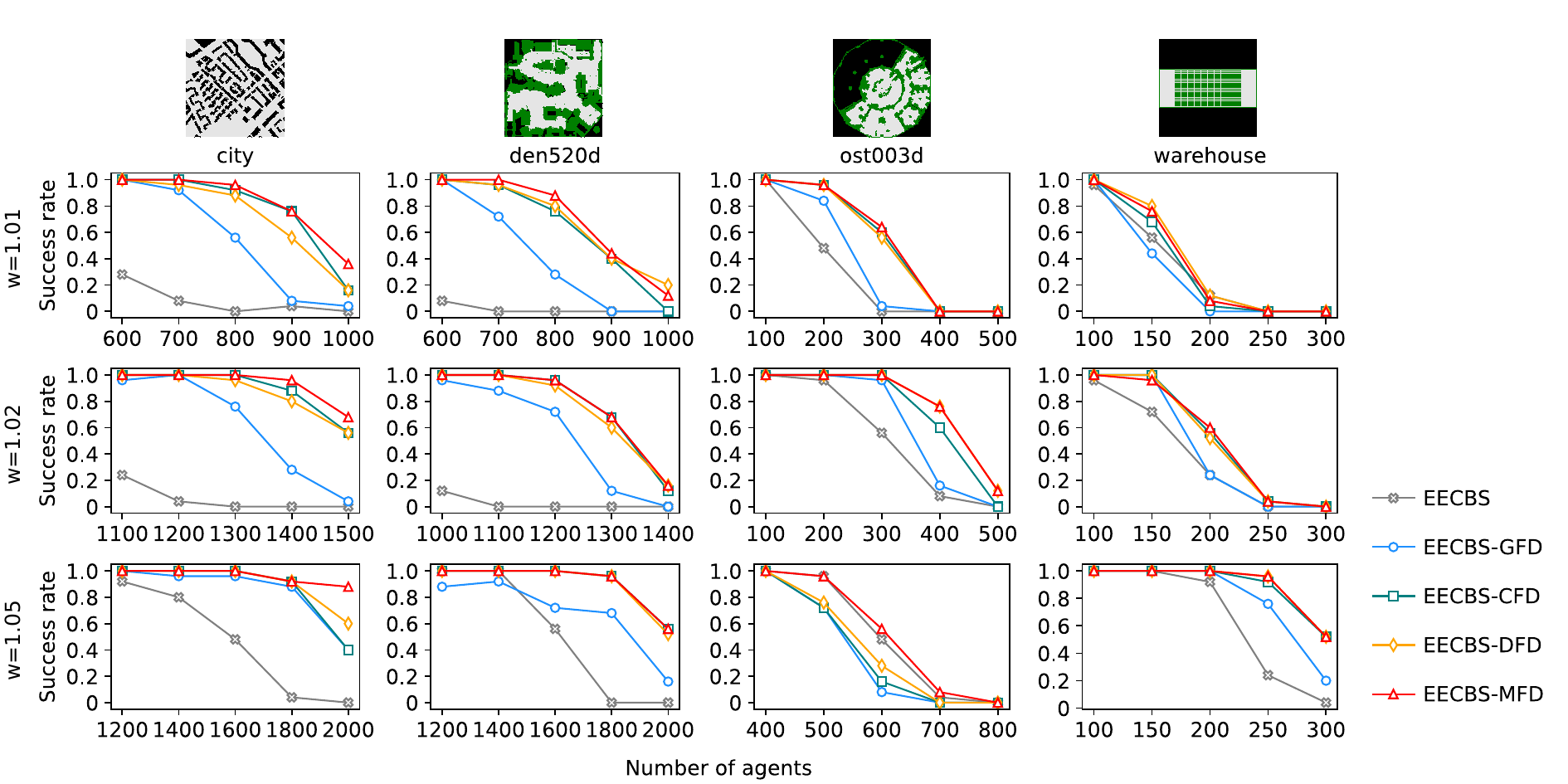}
    \caption{The success rate of EECBS, EECBS-GFD, EECBS-CFD, EECBS-DFD, and EECBS-MFD with suboptimality factor $w=\{1.01, 1.02, 1.05\}$ over MAPF instances with the same number of agents on the same graph.}
    \label{fig:succ_all}
\end{figure*}%

\begin{figure}
    \centering
    \includegraphics[width=0.85\linewidth]{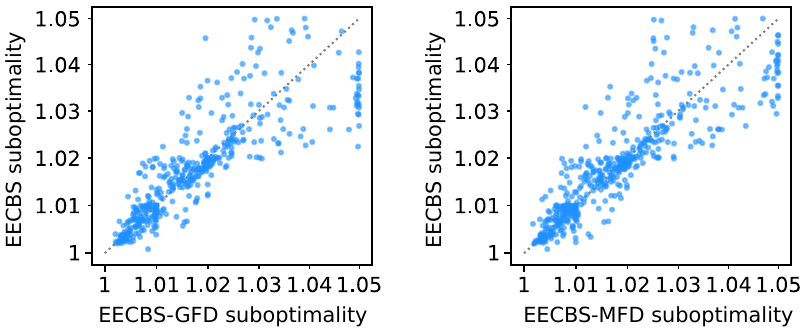}
    \caption{The global suboptimality of EECBS-GFD and EECBS-MFD in comparison to EECBS.}
    \label{fig:subopt-compare}
\end{figure}

We use the success rate to evaluate the efficiency of each EECBS variant and use the global suboptimality to evaluate the solution quality.
As shown in Figure~\ref{fig:succ_all}, EECBS-CFD and EECBS-DFD already outperform the state-of-the-art EECBS and EECBS-GFD, while EECBS-MFD improves the success rates further.
Figure~\ref{fig:subopt-compare} shows the global suboptimality of EECBS-GFD and EECBS-MFD in comparison to that of EECBS.
Each dot in the left (with respect to the right) figure represents a MAPF instance solved by both EECBS and EECBS-GFD (with respect to EECBS-MFD), where its coordinate corresponds to the respective suboptimality.
Among all MAPF instances solved by EECBS and EECBS-GFD, 56\% of them are the ones where EECBS-GFD has a higher global suboptimality than EECBS.
On the other hand, among all MAPF instances solved by EECBS and EECBS-MFD, 53\% of them are the ones where EECBS-MFD has a higher global suboptimality than EECBS, which is lower than that in comparison to EECBS-GFD.


\begin{table}[t]
\centering
\begin{tabular}{|l|rrrrr|}
\hline
          & 100           & 200           & 300           & 400           & 500           \\ \hline
EECBS     & 0.99          & 0.87          & 0.50          & 0.18          & 0.18          \\
EECBS-GFD & \textbf{1.00} & \textbf{1.00} & 0.97          & 0.18          & 0.13          \\
EECBS-MFD & \textbf{1.00} & \textbf{1.00} & \textbf{0.99} & \textbf{0.55} & \textbf{0.22} \\ \hline
\end{tabular}
\caption{The ratio between the depth and the number of expansion in a CT, averaged among all MAPF instances with the same number of agents in \texttt{ost003d} graph and $w=1.02$. Numbers in bold indicate the maximum ratio.}
\label{tab:dp_exp}
\end{table}%

To evaluate the progress that EECBS has made in finding a bounded-suboptimal solution, we define the depth of the CT.
If the EECBS finds a solution within the runtime limit, then the \textit{terminated CT node} is defined as the CT node containing the solution; otherwise, it is defined as the deepest expanded CT node.
Then, the \textit{depth} of a CT is defined as the number of CT nodes along the branch from the root CT node to the terminated CT node.
Table~\ref{tab:dp_exp} shows the ratio between the depth and the number of expansions in a CT among EECBS, EECBS-GFD, and EECBS-MFD.
The higher the ratio, the more \textit{focused} EECBS is, indicating that more CT node expansions are used in resolving conflicts and finding a bounded-suboptimal solution instead of switching between CT nodes on different branches.
When the number of agents is low, EECBS-GFD and EECBS-MFD are more focused than EECBS due to the usage of flex.
However, when the number of agents increases, the ratio of EECBS-GFD decreases due to the greedy mechanism, while EECBS-MFD still maintains the highest ratio.
In addition, Figure~\ref{fig:GFD_limitaion_globRatio} shows that EECBS-MFD has a higher GB ratio than EECBS-GFD.
Thus, MFD shows a better trade-off between increasing the SOC to reduce the number of conflicts and keeping those CT nodes globally bounded-suboptimal for future iterations.

\subsection{Case Study}\label{sec:instance}

\begin{figure}[t]
    \centering
    \includegraphics[width=0.7\linewidth]{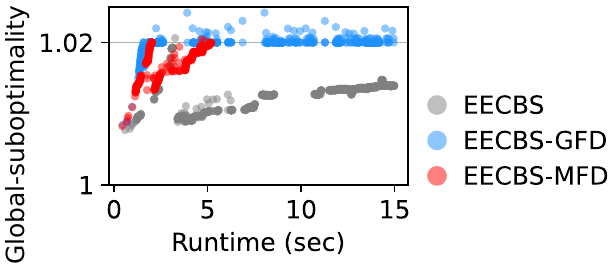}
    \caption{The global suboptimality of CT nodes of EECBS, EECBS-GFD, and EECBS-MFD versus the runtime during the search when solving a MAPF instance with 400 agents on \texttt{ost003d} graph with $w=1.02$. The numbers of generated CT nodes for EECBS, EECBS-GFD, and EECBS-MFD are respectively 1312, 608, and 350.}
    \label{fig:gen_instance}
\end{figure}%

\begin{figure}[t]
    \centering
    \includegraphics[width=0.8\linewidth]{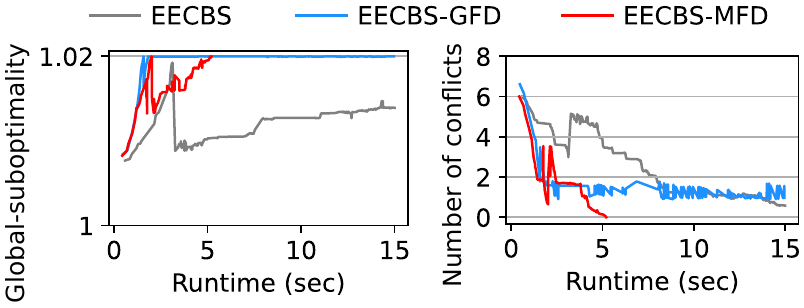}
    \caption{The global suboptimality (left) and the number of conflicts (right) of each expanded CT node during the search of EECBS, EECBS-GFD, and EECBS-MFD when solving the same MAPF instance as Figure~\ref{fig:gen_instance} with $w=1.02$. The numbers of CT node expansions for EECBS, EECBS-GFD, and EECBS-MFD are respectively 843, 395, and 305.}
    \label{fig:exp_instance}
\end{figure}%

\begin{figure}[t]
    \centering
    \includegraphics[width=0.8\linewidth]{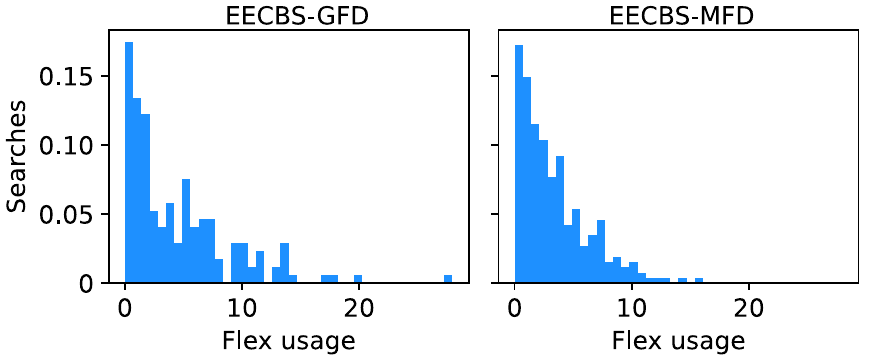}
    \caption{Histogram of flex usage with the non-negative maximum allowed flex for EECBS-GFD and EECBS-MFD, when solving the same MAPF instance as Figure~\ref{fig:gen_instance} with $w=1.02$. The number of searches is normalized, which is the percentage of the focal search processed within the range of flex usage.}
    \label{fig:flex_usg}
\end{figure}%

We run EECBS, EECBS-GFD, and EECBS-MFD, each with the suboptimality factor $w=1.02$, to solve a MAPF instance with 400 agents in \texttt{ost003d} graph under the runtime limit of 15 seconds. While solving this MAPF instance, both EECBS and EECBS-GFD timeout.
Figure~\ref{fig:gen_instance} shows the global suboptimality $C(N) / LB$ of all generated CT nodes $N$ versus runtime during the search, where each $LB$ value is recorded when each CT node is generated.
EECBS, requiring an individually bounded-suboptimal path for each agent, tends to generate CT nodes with low global suboptimality.
EECBS-GFD typically generates CT nodes $N$ with high global suboptimality due to the heavy usage of flex, which results in high SOC $C(N)$, and it may not expand such CT nodes as they may not be globally bounded-suboptimal (i.e., $C(N) / LB \nleq w$). 
On the other hand, MFD typically distributes less flex than GFD when generating a CT node.
Figure~\ref{fig:exp_instance} shows the global-suboptimality and the number of conflicts of each expanded CT node versus the runtime.
Although the number of conflicts in each expanded CT node of EECBS-MFD may be slightly higher than that of EECBS-GFD at the beginning of the search, EECBS-MFD makes a better trade-off than EECBS-GFD between the SOC and the number of conflicts when generating CT nodes and thus finds a bounded-suboptimal solution with zero conflicts under the runtime limit.
%
%
Figure~\ref{fig:flex_usg} shows the histogram of different ranges of flex usage when solving the MAPF instance with EECBS-GFD and EECBS-MFD, respectively. In comparison to EECBS-GFD, EECBS-MFD has a lower percentage of searches that use a large amount of flex (e.g., 0\% of searches with flex usage over 20).

\subsection{Evaluation on Congested MAPF instances}\label{sec:congested}
\begin{figure}[t]
    \centering
    \includegraphics[width=0.85\linewidth]{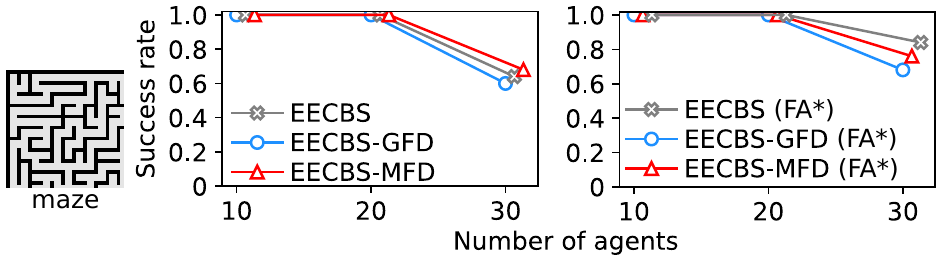}
    \caption{The success rate of EECBS, EECBS-GFD, and EECBS-MFD, with and without FA* search, when solving congested MAPF instances on a \texttt{maze} graph with a 120-second runtime limit and $w=1.05$.}
    \label{fig:congested_succ}
\end{figure}%

\begin{figure}[t]
    \centering
    \includegraphics[width=0.76\linewidth]{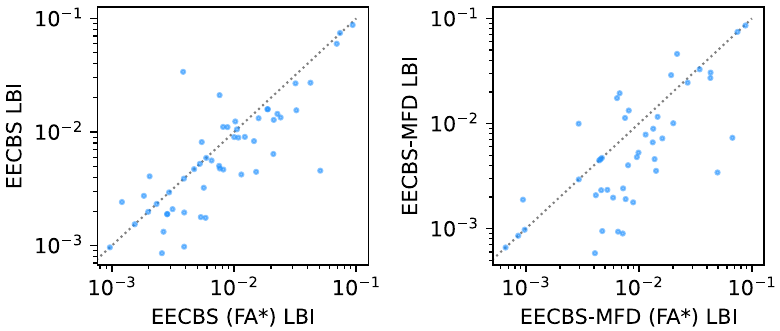}
    \caption{The LBI of EECBS, EECBS-GFD, and EECBS-MFD, with and without FA* search, when solving MAPF instances on \texttt{maze} graph with $w=1.05$.}
    \label{fig:congested_lbi}
\end{figure}%

Although the main goal of this work is to speed up EECBS in solving large-scale MAPF instances, we also test our approaches with $w=1.05$ on congested MAPF instances with 25 available random scenarios on a \texttt{maze} graph (\textit{maze-32-32-2}) of size $32 \times 32$.
Figure~\ref{fig:congested_succ} shows the success rates of different EECBS variants, where EECBS-GFD has the worst performance, and EECBS-MFD slightly outperforms EECBS.
However, with the help of FA* search, all algorithms improve their success rates, and EECBS still reaches the best performance.
Thus, how to distribute flex for congested MAPF instances remains an open question.
Figure~\ref{fig:congested_lbi} shows the lower bound improvement $LBI=(LB-LB_0) / LB_0$ over all MAPF instances used in Figure~\ref{fig:congested_succ}, where $LB$ is the lower bound on the optimal solution when the search terminates and $LB_0$ is the lower bound when the search begins (i.e., SOLB of the root CT node), from which the EECBS and EECBS-MFD have higher LBI with the help of FA* search.

\section{Conclusion}
We address the issue that distributing flex with a greedy mechanism may increase the SOC beyond $w \cdot LB$.
When finding a path for an agent, the objective of focal search is to quickly minimize the number of collisions while satisfying constraints.
Thus, we propose Conflict-Based Flex Distribution, Delay-Based Flex Distribution, and Mix-Strategy Flex Distribution. We also use FA* search to speed up EECBS for the congested MAPF instances. Our experiments show that our approaches are more efficient than the state-of-the-art EECBS.
Future work includes a better estimation of delays for constraints and a better mechanism that particularly targets congested MAPF instances.

\section{Acknowledgments}
The research at the University of California, Irvine and University of Southern California was supported by the National Science Foundation (NSF) under grant numbers 2434916, 2346058, 2321786, 2121028, and 1935712 as well as gifts from Amazon Robotics.

\bibliography{aaai25}

\end{document}